\documentclass[letterpaper, 10 pt, conference]{ieeeconf}
\IEEEoverridecommandlockouts

\usepackage{graphics}
\usepackage{amsmath}
\usepackage{amssymb}
\usepackage{mathtools}
\usepackage{xcolor}

\usepackage{url}
\usepackage{bbm}
\usepackage{pifont}
\usepackage{subcaption}
\usepackage{algorithm}
\usepackage{algorithmic}
\usepackage{cite}

\usepackage{todonotes}

\makeatletter
\newcommand\fs@betterruled{%
  \def\@fs@cfont{\bfseries}\let\@fs@capt\floatc@ruled
  \def\@fs@pre{\vspace{4pt}\hrule height.7pt depth0pt \kern2pt\vspace{-2pt}}%
  \def\@fs@post{\vspace{-2pt}\kern2pt\hrule\relax\vspace{-1pt}}%
  \def\@fs@mid{\vspace{-3pt}\kern2pt\hrule\kern2pt}%
  \let\@fs@iftopcapt\iftrue}
\floatstyle{betterruled}
\restylefloat{algorithm}
\makeatother

\newcommand{\argmax}{\mathop{\mathrm{argmax}}}

\newtheorem{proposition}{Proposition}

\newtheorem{definition}{Definition}

\title{\LARGE \bf
Secure Planning Against Stealthy Attacks via Model-Free Reinforcement Learning
}

\author{Alper Kamil Bozkurt, Yu Wang, and Miroslav Pajic
\thanks{This work is sponsored in part by the ONR under agreements N00014-17-1-2504, N00014-20-1-2745 and N00014-18-1-2374, AFOSR award number FA9550-19-1-0169, and the NSF CNS-1652544 grant.}
\thanks{Alper Kamil Bozkurt, Yu Wang, and Miroslav~Pajic are with Duke University, Durham, NC 27708, USA, {\tt\small \{alper.bozkurt, yu.wang094,  miroslav.pajic\}@duke.edu}}}

\begin{document}

\maketitle
\thispagestyle{empty}
\pagestyle{empty}

\begin{abstract}
We consider the problem of security-aware planning in an unknown stochastic environment, in the presence of attacks on control signals (i.e., actuators) of the robot. We model the attacker as an agent who has the full knowledge of the controller as well as the employed intrusion-detection system and who wants to prevent the controller from performing tasks while staying stealthy. 
We formulate the problem as a stochastic game between the attacker and the controller and present an approach to express the objective of such an agent and the controller as a combined linear temporal logic (LTL) formula. We then show that the planning problem, described formally as the problem of satisfying an LTL formula in a stochastic game, can be solved via model-free reinforcement learning when the environment is completely unknown. Finally, we illustrate and evaluate our methods on two robotic planning case~studies.
\end{abstract}

\section{Introduction} 
\label{section:introduction}

Security is an important concern for robotic %
systems working in critical applications. Malicious attacks on these systems can happen from various sources, exploiting vulnerabilities in, for instance,  system sensing (e.g., GPS~\cite{kerns2014unmanned,psiaki_AttackersCanSpoof_2016,kwon2016real} or automotive speed sensors~\cite{shoukry_NoninvasiveSpoofingAttacks_2013}) or software design~\cite{barthe_FacetsSoftwareDoping_2016,chowdhury2019survey}. To prevent %
such anomalous behaviors, a common approach is  to incorporate an intrusion detection system (IDS) into the overall system design. These components %
monitor in runtime key parts of the system, %
raising alarms when unexpected changes or behaviours, potentially caused by attacks, are~observed.

Although the use of IDS has significantly enhanced security guarantees in robotics, by limiting available attack vectors, 
sophisticatedly crafted attacks that are stealthy to IDS (i.e., undetected by the IDS) can still have significant impact on system performance. Stealthy attacks for control systems have been extensively studied. Using knowledge of the system models, various kinds of stealthy attacks have been designed, such as replay attacks (e.g.,~\cite{mo_SecureControlReplay_2009}), covert attacks (e.g.,~\cite{smith_CovertMisappropriationNetworked_2015}), zero dynamic attacks (e.g.,~\cite{teixeira_RevealingStealthyAttacks_2012}), as well as attacks for specific types of IDSs (e.g.,~\cite{mo2010false,kwon2014analysis,jovanov_tac19}). Also, without knowledge of the system models (i.e., black/grey-box attacks), machine learning-based methods have been recently used to %
design stealthy attacks~\cite{li_ConAMLConstrainedAdversarial_2020,zizzo_AdversarialMachineLearning_2019,feng_DeepLearningbasedFramework_2017}. %

In this work, we focus on defending against stealthy actuation attacks (i.e., on control signals) in robotic planning, where the dynamics are typically captured by finite-state probabilistic models. Specifically, we provide a reinforcement learning (RL)-based framework for security-aware design of control strategies that maximizes resilience of the robotic task to adversarial actions. To achieve this, we adopt a game-theoretic approach and capture the interaction between the (high-level) system controller and the attacks as turn-based \emph{stochastic games}~\cite{shapley1953}, which is an extension of Markov decision processes (MDPs) with~two~players. %

The objectives for robotic planning tasks with temporal properties %
are commonly expressed  by  formulas %
in linear temporal logic (LTL) (e.g.,~\cite{kress-gazit2007,guo2013revising}). 
The fulfillment of an LTL objective %
may not only depend on the current state (e.g., safety) but also on the whole system execution (e.g., visiting a state infinitely~often for surveillance). Hence, our security-aware planning focuses on design of control strategies that maximize the worst-case (due to the attacker actions) probability that the given LTL task specification $\varphi_{\textsc{task}}$ is~satisfied.

For the security measure, we consider a very general class of logic-based IDSs that monitor the system execution and trigger alarm if a given LTL formula $\varphi_{\textsc{ids}}$ is satisfied. This includes the common window-based IDSs that trigger alarm when there are several unexpected system transitions within a time window or all past moves in general~\cite{bonakdarpour2018opportunities,medhat2015runtime,havelund2002synthesizing,moonjookim_FormallySpecifiedMonitoring_1999}. 
Following this problem setup, we consider the defense against two types of stealthy attacks with an increasing level of aggressiveness. In Case I, the attacker tries to sabotage the task objective without triggering alarm from the IDS. %
In Case II, the attacker can take the risk of being detected by IDS if the probability of finally sabotaging the initial LTL objective is maximized. For these types of attacks, we solve the game between the attacker and controller, and derive the optimal defense strategies for the controller.

Furthermore, to allow the use of our security-aware planning methodology for robotic systems with unknown models, %
we adopt a model-free RL approach to solve the game between the attacker and controller. %
Although solving stochastic games with temporal logic objectives by model-based methods is well-studied, there are few works on model-free methods that do not depend on the knowledge on the transition probabilities. By exploiting recent results from~\cite{bozkurt2020}, to the best of our knowledge, we introduce the first model-free RL method for stochastic games that are related to secure planning for LTL objectives. Finally, with two case studies focused on robotic surveillance and task sequencing, we demonstrate applicability of our methodology for security-aware robotic~planning.

\section{Preliminaries and Problem Statement}
\label{section:problem_formulation}

\subsection{Stochastic Games}
Stochastic games provide a powerful framework to model and reason about %
behavior of multiple self-interested agents in stochastic environment, capturing both nondeterministic and stochastic transitions. We focus on a scenarios where there are only two agents who are strictly competitive (zero-sum), which means if one agent wins, then the other loses. %
The controller (i.e., Player~1) wants to perform a given task, whereas the attacker (i.e., Player~2) wants to prevent~that.

\begin{definition}[Stochastic Games] \label{def:sg}
A (labeled turn-based) two-player stochastic game is a tuple $\mathcal{G}=(S,(S_\mu,S_\nu,S_p), \allowbreak A, P, s_0, \textnormal{AP}, L)$, where
$S = S_\mu \cup S_\nu \cup S_p$ is a finite set of states, with $S_\mu$, $S_\nu$ and $S_p$ being disjoint; $S_\mu$ and $S_\nu$ are the sets of states where the controller or the attacker, respectively, choose actions; 
$S_p$ is the set of stochastic states;
$s_0 \in S$ is an initial state;
$A$ is a finite set of actions and $A(s)$ denotes the set of actions that can be taken by the controller or the attacker if $s\in S_\mu$ and $s \in S_\nu$ respectively, and denotes a single dummy action for all $s\in S_p$;
$P: S \times A \times S \to [0,1]$ is a transition probability function such that $P(s, a, s') \in \{0,1\}$ for $s \in S_\mu \cup S_\nu$, and $\sum_{s' \in S} P(s, a, s')=1$ if $a\in A(s)$, and $0$ otherwise for all $s \in S$;
\textnormal{AP} is a finite set of atomic propositions; and
$L\hspace{-2pt}:\hspace{-2pt}S\to 2^{\textnormal{AP}}$ is a labeling~function.
\end{definition}

In a stochastic game $\mathcal{G}$, the successor of a state $s$ is chosen by the controller if $s\in S_\mu$ or by the attacker if $s\in S_\nu$; and if $s\in S_p$, it is randomly chosen according to the probability distribution $P(s,A(s),\cdot)$. A \textit{path} $\sigma=s_0s_1,\dots$ is the infinite sequence of the visited states. For simplicity, we denote the state $s_t$ by $\sigma[t]$ and the suffix $s_{t}s_{t+1}\dots$ by $\sigma[t{:}]$. 

{The strategies of the controller and attacker are captured by a function that maps a finite prefix, i.e., a history of visited states, to a probability distribution over the actions that can be taken in the last state. Here, we focus on pure and finite-memory strategies as %
we will show later (in Sec.~\ref{section:rl})) that they suffice for security-aware planning for LTL tasks~\cite{chatterjee2012}.}

\begin{definition}%
A \emph{finite-memory strategy} for a game $\mathcal{G}$ is a tuple $\pi=(M,\Delta,\alpha,m_0)$ where: $M$ is a finite set of modes; $\Delta\hspace{-2pt}: M\hspace{-2pt}\times\hspace{-2pt} S \hspace{-2pt}\to \hspace{-2pt}M$ is a transition function; $\alpha\hspace{-2pt}: M\times S\setminus S_p \to A$ is a function that maps the current mode $m\in M$ and state $s$ to an action in $A(s)$; and $m_0$ is an initial mode.
A \emph{controller strategy} $\mu$ is a finite-memory strategy %
that only maps states in $S_\mu$ to actions (i.e., $\alpha: M{\times} S_\mu {\to} A$). Similarly, an \emph{attacker strategy} $\nu$ is a finite-memory strategy where $\alpha: M{\times} S_\nu {\to} A$.
\end{definition}

An optimal controller strategy is defined as a strategy under which the probability that the controller successfully performs a given task is maximized in the worst-case. 

\subsection{Capturing Temporal Specifications}
LTL extends the propositional logic with temporal modalities: next ($\bigcirc$) and until ($\textsf{U}$). LTL formulas are constructed via a recursive combinations of Boolean operators and temporal modalities using the following syntax~\cite{baier2008}: 
\begin{align}
    \varphi \coloneqq  \mathrm{true} \mid a \mid \varphi_1 \wedge \varphi_2 \mid \neg \varphi \mid \bigcirc \varphi \mid \varphi_1 \textsf{U} \varphi_2, ~ {a\in\textnormal{AP}}, \label{eq:ltl}
\end{align}
where $\textnormal{AP}$ is a set of atomic propositions.

LTL formulas specify properties of infinite paths of %
games. %
In addition to the satisfaction of the standard logical operations, the fulfillment of an LTL formula $\varphi$ on a path $\sigma$ of game $\mathcal{G}$, denoted by $\sigma {\models} \varphi$, is recursively defined~as: 
$\sigma$ satisfies an atomic proposition $a$, if $a {\in} L(\sigma[0])$;
$\sigma$ satisfies $\bigcirc \varphi$ if $\sigma[1{:}]$ satisfies  $\varphi$; 
and finally, 
$\sigma {\models} \varphi_1 \textsf{U} \varphi_2$, if $\exists i.\ \sigma[i] {\models} \varphi_2$ and $ \forall j{<}i. \ \sigma[j] \models \varphi_1$.
Also, we write %
(eventually) $\lozenge \varphi $ for $ \mathrm{true}\ \textsf{U}\ \varphi$, (always) $\square \varphi $ for $ \neg (\lozenge \neg \varphi)$, and
\begin{align}
    \lozenge^{\leq k}\varphi \coloneqq \bigcirc \varphi \vee \bigcirc \bigcirc \varphi \vee \dots \vee \underbrace{\bigcirc \bigcirc \dots \bigcirc}_{k\text{ times}} \varphi.
\end{align}

Deterministic Rabin automata (DRAs) offers a systematic way for LTL model checking. An LTL formula can be represented by a DRA that accepts a path if and only if the path satisfies the LTL formula; %
the acceptance conditions of DRAs are defined for infinite paths \cite{baier2008}.
\begin{definition}%
\label{def:dra}
A DRA is a tuple $\mathcal{A}= (Q,\Sigma, \delta, q_0, \textnormal{Acc})$ where
$Q$ is a finite set of states; 
$\Sigma$ is a finite alphabet; 
$\delta: Q \times \Sigma \to Q$ is the transition function; 
$q_0 \in Q$ is an initial state; 
$\textnormal{Acc}$ is a set of $k$ accepting pairs $\{(C_i, B_i)\}_{i=1}^k$ such that $C_i, B_i \subseteq Q$ such that an infinite path $\sigma$ is accepted if 
    \begin{align}
        \exists i: \ \text{inf}(\sigma)\cap C_i =\varnothing \ \wedge \ \text{inf}(\sigma) \cap B_i \neq \varnothing, \label{eq:dra_acc}
    \end{align}
where $\text{inf}(\sigma)$ is the set of states visited infinitely often during the execution induced by the labels of~$\sigma$.
\end{definition}

A pair $(\mu,\nu)$ of a controller strategy $\mu$ and an attacker strategy $\nu$ in a game $\mathcal{G}$ induces a Markov chain (MC), denoted by $\mathcal{G}_{\mu,\nu}$. The probability that a path $\sigma$ sampled from an MC $\mathcal{G}_{\mu,\nu}$ satisfies the LTL formula $\varphi$ is defined as:
\begin{align}
    Pr_{\mu,\nu}(\mathcal{G} \models \varphi) \coloneqq Pr_{\sigma \sim \mathcal{G}_{\mu,\nu}} \big\{ \sigma \mid \sigma \models \varphi \big\}.
\end{align}
The objective in to find a controller strategy $\mu_\varphi$ such that
\begin{align}
    \mu_\varphi = \argmax\nolimits_\mu \min\nolimits_\nu Pr_{\mu,\nu}(\mathcal{G} \models \varphi) \label{eq:optimal_objective_strategy}
\end{align}
where $\mu$ and $\nu$ denote any finite-memory controller and attacker strategies, respectively. {Note that $\mu_\varphi$ might not be unique}; in that case, slightly abusing the notation, we let $\mu_\varphi$ to denote the set of all such controller strategies.

\subsection{Problem Statement: Security-Aware Planning}
This work is motivated by the reported susceptibility of autonomous systems to attacks~\cite{pajic_csm17,kerns2014unmanned,kwon2016real}.
Several~methods are proposed to deal with attacks on system sensing, both for low-level control (e.g.,~\cite{pajic_tcns17,bezzo_iros14,chang2018secure}), and path planning (e.g., \cite{elfar_cav19, elfar_icra19}). %
However, 
to the best of our knowledge,~no methods exist to design controller strategies for uncertain~environments, such that the derived controllers are maximally resilient to  actuation attacks %
-- this is the focus of this~work. 

Specifically, we consider scenarios where a smart attacker can modify control actions with the goal of preventing the robot from performing the given task, captured by an LTL objective $\varphi_{\textsc{task}}$. We assume that the system has an Intrusion Detection System (IDS) used to detect system anomalies; thus, the attacker's objective is also to remain \emph{stealthy}. This can be effectively achieved  by injecting non-aggressive and incremental control perturbations that exploit probabilistic uncertainties in robot motion (captured by the system model). 

Consequently, in this work, we address the problem of \emph{synthesizing security-aware controller strategies} that maximize the worst-case (i.e., even for the most damaging attacker actions) probability of satisfying the LTL mission objectives $\varphi_{\textsc{task}}$. In addition, we assume that the controller does not know model of the system. Hence, the problem is formulated as a stochastic game $\mathcal{G}$ where the transition probabilities and the topology of the game are unknown, and our goal is to design a \emph{model-free} reinforcement learning (RL)-based framework for synthesis of such security-aware controller strategies. Note that since no attacks on sensing are considered in this work, both the controller and the attacker have the knowledge of the current state of the robot.

In the rest of the paper, we first show how to `combine' the attacker's stealthiness constraint with the control objective (Sec.~\ref{section:ltl}); this enables us to formulate this problem as a control synthesis problem from such combined LTL formulas for stochastic games. Then, we show how model-free RL can be used to derive such maximally resilient (i.e., security-aware) controller strategies (Sec.~\ref{section:rl}), 
by exploiting suitably crafted rewards and discounting based on the LTL~formula.

\section{LTL Specification of Security-Aware Controller Objectives}\label{section:ltl}

In this section, %
we show how the IDS' triggering condition %
can be captured by a winning LTL specification for the game.

\subsection{Controller Objective as Winning Condition}

The controller's main objective is to %
successfully perform a given task $\varphi_\textsc{task}$. %
As with previous works~e.g.,~\cite{fainekos_TemporalLogicMotion_2009,kress-gazit_SynthesisRobotsGuarantees_2018}, we capture the task $\varphi_\textsc{task}$ by LTL; this includes common planning tasks such as avoidance ($\square \neg \texttt{\small unsafe}$), 
liveness/recurrence ($\square \lozenge \phi$), 
persistence ($\lozenge \square \texttt{\small safe}$), 
coverage of other tasks ($\lozenge \phi_1 \wedge \lozenge\phi_2 \wedge \dots \wedge \lozenge\phi_n$) or 
sequencing of other tasks ($\lozenge (\phi_1 \wedge \lozenge(\phi_2 \wedge \dots \wedge \lozenge\phi_n)\dots)$).

In addition, we assume that after the attack is detected, the attacker is no longer capable of attacking, and thus additional requirement for the controller can be stated as attack detection %
$\varphi_\textsc{ids}$ -- we discuss in detail how to construct LTL formula $\varphi_\textsc{ids}$ in Sec.~\ref{subsection:ids}. Therefore, 
the winning condition for the controller can be captured in LTL as:%
\begin{align}
    \varphi_\textsc{win} = \varphi_{\textsc{ids}} \vee \varphi_{\textsc{task}}. \label{eq:win}
\end{align}

There are several important implications of $\varphi_\textsc{win}$ in \eqref{eq:win}. First, if the controller can always satisfy $\varphi_{\textsc{task}}$ regardless of being under an attack, then the controller does not need to use an IDS. Similarly, if $\varphi_{\textsc{ids}}$ can never be satisfied, the IDS mechanism does not help the controller at all.
Second, $\varphi_\textsc{ids}$ should be satisfied only if there is an actual attack; otherwise, with the above winning condition, the controller has an incentive to look for situations where a false alarm can be raised, in order to win the game. Third, detecting an attack should not be enough to win the game as the controller still needs to be able to perform the task. In addition, the attacker might risk being detected for an attack that may cause the controller to end up in a state from which the task cannot be %
satisfied. Hence, %
$\varphi_\textsc{ids}$ needs to be extended to also cover these cases, as done in the rest of the section. %

\subsection{Intrusion Detection in Stochastic Environments} \label{subsection:ids}

An IDS monitors system evolution, %
continuously evaluating  if %
the observed behavior deviates from the expected ones. Analysis of several IDS is recently performed in the context of security-aware control (e.g.,~\cite{mo2010false,kwon2014analysis,jovanov_tac19}). When executing in uncertain environments (e.g., modeled as stochastic games), such IDS commonly provide probabilistic guarantees, and can potentially cause (e.g., fixed-rate) false alarms.

For example, consider a planning scenario where a %
robot may not move in the intended direction with a probability of at most $p_\varepsilon$. %
In practice, effectively monitoring such probabilistic behaviors is mapped into 
counting the number of unintended moves in a pre-specified time-window or in all past moves in general~\cite{bonakdarpour2018opportunities,medhat2015runtime,havelund2002synthesizing,moonjookim_FormallySpecifiedMonitoring_1999}. %
Consider e.g., that for every four-time-step window, an alarm should be raised if the count is larger than 1; in this case, even without an attacker, %
when $p_\varepsilon=0.1$, the likelihood of a false alarm %
is $0.0523$ and the expected number of steps before an alarm %
is  $47$. 

For any $p_\varepsilon{>}0$, any such window-based IDS eventually raises alarm, even without attacker. 
Thus, when considering large-enough (i.e., infinite) paths, if a false alarm is considered as an attacker being detected, then the controller wins the game by satisfying $\varphi_\textsc{ids}$ and thereby $\varphi_\textsc{win}$, even when the attacker is inactive. %
On the other hand, to deal with the false alarms in practice, the system would perform a thorough %
inspection on the next steps (with a significant overhead) such that any attacks %
are detected based on the employed attack vectors (i.e., actual attack surfaces~\cite{lesi_rtss17,hasan2016exploring}). 
If the IDS does not detect any attack for a pre-determined number of time steps during the close inspection, the system goes back to its default execution mode.

Consequently, following~\cite{bonakdarpour2018opportunities,medhat2015runtime,havelund2002synthesizing,moonjookim_FormallySpecifiedMonitoring_1999}, we provide an LTL formulation that can capture a large class of IDSs that satisfy the aforementioned constraints. Consider the following window-based IDS described by the LTL formula:
\begin{align}
    \varphi_{\textsc{alarm}} = \lozenge\big(\texttt{\footnotesize anomaly} \wedge \bigcirc \lozenge^{\leq m} \texttt{\footnotesize anomaly}\big) \label{eq:alarm}
\end{align}
where \texttt{\small anomaly} is the label of unexpected executions. The alarm formula $\varphi_{\textsc{alarm}}$ can be interpreted as: \emph{raise an alarm if an anomaly occurs and after that another one occurs within the next $m+1$ time steps}. Such formula can be easily extended to count/allow for more anomalies within~a different window size, by adding nested $\bigcirc\lozenge^{\leq m_i}$, as in e.g.,
\begin{align}
\varphi_{\textsc{alarm}_2} {=} \lozenge(\texttt{\footnotesize anomaly} {\wedge} {\bigcirc} {\lozenge}^{\leq m_1} ( \texttt{\footnotesize anomaly} {\wedge} {\bigcirc} {\lozenge}^{\leq m_2} \texttt{\footnotesize anomaly}) )\notag
\end{align}
We can also count consecutive anomalies by setting~some $m_i$ to zeros. Many IDS mechanisms can be inherently~described as a reachability property, which can be %
specified by the logical operators and %
temporal modalities $\bigcirc$, $\lozenge$ and $\lozenge^{\leq k}$.

After an alarm occurs, the system switches to a high-alert mode and can detect
any new attacks. %
This is %
captured~by 
\begin{align}
    \varphi_{\textsc{detect}} = \lozenge^{\leq n} \texttt{\footnotesize attack}, \label{eq:detect}
\end{align}
where \texttt{\small attack} is the label of any transition considered as an attack by the IDS. The detection formula $\varphi_{\textsc{detect}}$ simply means that any attack within $n+1$ time steps will be detected. %

We can integrate a detection formula into an alarm formula to obtain a combined IDS formula. This is achieved by adding $\wedge \bigcirc \varphi_{\textsc{detect}}$ next to the last \texttt{\small anomaly} in the alarm formula that triggers the alarm. For example, if we integrate $\varphi_{\textsc{detect}}$ in \eqref{eq:detect} to $\varphi_{\textsc{alarm}}$ in \eqref{eq:alarm}, we obtain:
\begin{align}
\label{eq:ids}
    \varphi_{\textsc{ids}} = \lozenge(\texttt{\footnotesize anomaly} \ \wedge \bigcirc \lozenge^{\leq m} (\texttt{\footnotesize anomaly} \wedge \notag \bigcirc \lozenge^{\leq n} \texttt{\footnotesize attack} )). 
\end{align}
Note that the negation of $\varphi_{\textsc{ids}}$ nicely reflects the behavior of a stealthy attacker. An attacker with the objective $\neg\varphi_{\textsc{ids}}$ wants to stay undetected at all costs and is reluctant to frequently take actions to avoid triggering alarm, so that the attacks are `hidden' within the stochastic behavior of the environment.

One problem with this formulation is that it does not express the cost of carrying out the  close system inspection. %
This, unfortunately, gives an incentive for the controller to trigger a false alarm. For example, the controller may try to move to the states that exhibit a high degree of randomness to increase its chance to switch to the high-alert mode, after which it can continue to perform its task without the fear of being under attack for the predefined number~of~steps. We discuss a strategy to overcome this problem in Sec.~\ref{section:rl}. Furthermore, even if the attacker is detected, the robot still needs to perform its task; we address this~as~follows.

\subsection{Performing Tasks After Attack Detection}\label{section:after_attack}

The attacker's goal is to %
prevent satisfaction of the control task.
Thus, %
it may be beneficial for the attacker to %
launch an attack even during the high-alert mode (i.e., at the cost of being detected and eliminated),
if 
the controller could no longer %
fulfill its task after the attack. 

This can %
be embedded into the previously described LTL formula $\varphi_{\textsc{ids}}$ 
by replacing \texttt{\small attack} in $\varphi_{\textsc{ids}}$ with $\texttt{\small attack} \wedge \bigcirc \lozenge \texttt{\small attack}$. For example, the previous $\varphi_{\textsc{ids}}$ %
becomes
\begin{align*}
\varphi_{\textsc{ids}} {=} {\lozenge}(&\texttt{\scriptsize anomaly} {\scriptstyle \wedge}
{\scriptstyle \bigcirc} {\scriptstyle \lozenge}^{\scriptscriptstyle \leq m} (\texttt{\scriptsize anomaly} {\scriptstyle \wedge} {\scriptstyle \bigcirc} {\scriptstyle \lozenge}^{\scriptscriptstyle \leq n} (\texttt{\scriptsize attack} {\scriptstyle \wedge} {\scriptstyle \bigcirc} {\scriptstyle \lozenge} \texttt{\scriptsize attack}))).
\end{align*}

Although this might seem as giving a second chance to the attacker after being detected, (s)he cannot attack for a second time because that would result in the attacker immediately losing the game.
Another concern with this formulation is that the IDS must be in the high-alert mode all the time after the first attack to be able to observe a second attack; however, this is not necessary because the second attack cannot happen in practice based on the assumption that the attacker is eliminated after the first attack. Hence, this IDS formulation allows the scenarios where the attacker is detected but might still win the game if the controller fails to perform its task.
\section{Model-Free Learning from the Winning LTL Specifications}%
\label{section:rl}

In this section, we use model-free reinforcement learning to find optimal controller strategies that maximize the (worst-case) probability of satisfying the $\varphi_\textsc{win}$ from~\eqref{eq:win} -- i.e., 
\begin{equation}
    \mu_{\varphi_\textsc{win}} = \argmax\nolimits_\mu \min\nolimits_\nu Pr_{\mu,\nu}(\mathcal{G} \models \varphi_\textsc{win}). \label{eq:winning_strategy}
\end{equation}
Our method requires no knowledge of the transition probabilities or the topology of the game $\mathcal{G}$, which models the interaction between the controller and the attacker. 

\subsection{From LTL to Discounted Rewards for Model-Free RL}

Model-free RL provides an efficient way to search for optimal strategies without reconstructing the explicit model (e.g., the transition probabilities) of the stochastic game from samples. 
However, existing model-free RL methods can only be used for cumulative rewards associated with the states (or transitions) and 
cannot be directly used for LTL objectives on stochastic games.
Thus, we convert the LTL formula $\varphi_\textsc{win}$ into a cumulative discounted reward below.

Let $G(\sigma)$ denote the return, the sum of discounted rewards of a path $\sigma$ of $\mathcal{G}$, defined as:
\begin{align}
    G(\sigma) &\coloneqq \sum\nolimits_{i=0}^{\infty} \left( \prod\nolimits_{j=0}^{i-1} \Gamma(\sigma[j]) \right) \cdot R(\sigma[i]); \label{eq:return}
\end{align} 
here, $R:S\to[0,1)$ and $\Gamma: S \to (0,1)$ are the state-based reward and discount functions. We make the convention that $\prod_{j=0}^{-1} \coloneqq 1$. 
The goal of model-free RL is to find a controller strategy that maximizes the expected return in the worst case:
\begin{align}
    \mu_* = \argmax\nolimits_\mu \min\nolimits_\nu \mathbb{E}_{\sigma \sim \mathcal{G}_{\mu,\nu}} \left[G(\sigma) \right]. \label{eq:optimal_discounted_reward_strategy}
\end{align}
Our goal, now, becomes to design the functions $R$ and a $\Gamma$ in such a way that $\mu_*$ and $\mu_{\varphi_\textsc{win}}$ from~\eqref{eq:winning_strategy} %
become equal.

Recently, there has been an increasing interest in developing of such model-free RL methods %
\cite{bozkurt2019,bozkurt2020,hahn2019,hahn2020}. %
Here, we adopt the method introduced in~\cite{bozkurt2020}, which takes a turn-based stochastic game $\mathcal{G}$ where $P$ is fully unknown, an LTL specification of the winning condition $\varphi_\textsc{win}$, a discount factor $\gamma$, and learns the controller strategy $ \mu_{\varphi_\textsc{win}}$. 
For the discount factors sufficiently close to $1$, if the LTL formula can be translated into a DRA with one accepting pair, the RL algorithm is guaranteed to converge to an optimal controller strategy, otherwise it converges to a controller strategy with a satisfaction probability above an established lower bound. 

The idea is to augment the state space such that it is sufficient to consider only the pure and memoryless strategies so that a state-based $R$ and $\Gamma$ can be defined. Specifically, using the method from~\cite{bozkurt2020}, we construct a product game from the original game $\mathcal{G}$ and the DRA $\mathcal{A}_{\varphi_\textsc{win}}$ obtained from the LTL formula 
$\varphi_\textsc{win}$, thus reducing the winning criteria to the satisfaction of the Rabin acceptance condition from~\eqref{eq:dra_acc}. 
\begin{definition}[Product Games] \label{def:product} A product game of  $\mathcal{G}= (S, (S_\mu,S_\nu,S_p), A, P, s_0, \textnormal{AP}, L)$, a labeled turn-based stochastic game, and a DRA $\mathcal{A}=(Q,2^\textnormal{AP}, \delta, q_0, \text{Acc})$, is the tuple $\mathcal{G}^\times = (S^\times, \allowbreak (S_\mu^\times,S_\nu^\times,S_p^\times), \allowbreak  A^\times, P^\times, s_0^\times, \textnormal{Acc}^\times)$ where:
\begin{itemize}
    \vspace{-1pt}
    \item $S^\times=S\times Q$ is the set of augmented states, and the initial state $s_0^\times$ is $\langle s_0,q_0 \rangle$;
    \item $S_\mu^\times=S_\mu\times Q$ is the set of augmented controller states;
    \item $S_\nu^\times=S_\nu\times Q$ is the set of augmented attacker states;
    \item $S_p^\times=S_p\times Q$ is the set of augmented stochastic states;
    \item $A^\times=A$ is the set of actions where $A^\times(\langle s, q\rangle) = A(s)$ for all $s\in S, q \in Q$; 
    \item $P^\times:S^\times \times A^\times \times S^\times \to [0,1]$ is the transition function:
    \begin{align}
        &P^\times(\langle s,q \rangle,a,\langle s',q'\rangle) {=} 
        \begin{cases}
            P(s,a,s') & \textrm{if } q' {=} \delta(q,L(s)) \\
            0 & \textnormal{otherwise}
        \end{cases} \notag 
    \end{align}

    \item $\textnormal{Acc}^\times$ is a set of $k$ accepting pairs $\{(C_i^\times, B_i^\times)\}_{i=1}^k$ where $C_i^\times = C_i \times Q$ and $B_i^\times = B_i \times Q$.
\end{itemize}
\end{definition}

When there is only a single pair Rabin in the acceptance condition (i.e., $\text{Acc}=\{(C,B)\}$), there is always a pure and memoryless optimal strategies for both the controller and the attacker. This allows to construct simple $R$ and $\Gamma$ functions in~\eqref{eq:return}. For example, a small positive reward can be received whenever a state in $B$ visited since some states in $B$ need to be visited infinitely often to win the game. Similarly, to discourage visiting the states in $C$ infinitely many times, future rewards can be heavily discounted. 
Consequently, Algorithm~\ref{alg:a1} summarizes the steps of our approach.

\setlength{\textfloatsep}{12pt}
\begin{algorithm}[!t]
\begin{algorithmic}
{\small
\STATE \textbf{Input:} LTL formula $\varphi_\textsc{win}$, stochastic game $\mathcal{G}$
\STATE Translate $\varphi_\textsc{win}$ to a DRA $\mathcal{A}_{\varphi_\textsc{win}}$
\STATE Construct the product $\mathcal{G}^\times$ of $\mathcal{G}$ and $\mathcal{A}_\varphi$
\STATE Learn the optimal state-action values $Q_*$ for $\mathcal{G}^\times$ using minimaxQ
\STATE Obtain a greedy strategy $\mu^\times_*$ from $Q_*$
\RETURN a finite-memory controller strategy $\mu_*$ derived from $\mu^\times_*$
}
\end{algorithmic}
\caption{Model-free RL for security-aware control synthesis from LTL specifications.}\label{alg:a1}
\end{algorithm}

\begin{proposition} \label{proposition}
Consider a stochastic game $\mathcal{G}$ and an LTL formula $\varphi_\textsc{win}$ that can be translated to a DRA with a single accepting pair. There exists $\gamma'$ such that for all $\gamma\in(\gamma' , 1)$, %
Algorithm~\ref{alg:a1} converges and returns a pure finite-memory controller strategy satisfying \eqref{eq:winning_strategy}, if the reward $R_\gamma$ and the discount $\Gamma_\gamma$ functions are defined as:

\vspace{-12pt}
\small
\begin{equation*}
    R_\gamma(s^\times) \coloneqq\hspace{-4pt} \begin{cases}
        1\hspace{-2pt}-\hspace{-2pt}\gamma_B(\gamma), & \hspace{-6pt} s^\times \in B^\times \\
        0, &\hspace{-6pt} s^\times \notin B^\times
    \end{cases}, ~~
        \Gamma_\gamma(s^\times)\coloneqq \hspace{-4pt} \begin{cases}
        \gamma_B(\gamma), & \hspace{-6pt} s^\times \in B^\times \\
        \gamma_C(\gamma), & \hspace{-6pt} s^\times \in C^\times \\
        \gamma, &\hspace{-6pt} \textrm{otherwise}
    \end{cases}  \label{eq:reward_discount}
\end{equation*}
\normalsize
where $\gamma_B$ and $\gamma_C$ are functions of $\gamma$ satisfying
\begin{align} \label{eq:gamma_lim}
    \lim_{\gamma \to 1^-} \frac{1 - \gamma}{1 - \gamma_B (\gamma)} = \lim_{\gamma \to 1^-} \frac{1 - \gamma_B(\gamma)}{1 - \gamma_C (\gamma)} = 0. 
\end{align}
\end{proposition}

\begin{proof}
The proof immediately follows from the proof of Theorem~1 in~\cite{bozkurt2020}, showing that under a strategy pair, as $\gamma \to 1^-$, the expected return of a path goes to the probability that the path satisfies the Rabin acceptance condition. 
Finally, the obtained strategy in the product game satisfying~\eqref{eq:optimal_discounted_reward_strategy},
induces a finite-memory controller strategy (in the initial game $\mathcal{G}$) satisfying~\eqref{eq:winning_strategy}, where the the modes are the states of the DRA $\mathcal{A}_{\varphi_\textsc{win}}$ with the same transition function. 
\end{proof}

Algorithm~1 can be generalized to situations where the DRA $\mathcal{A}_{\varphi_\textsc{win}}$ has
multiple Rabin pairs, using the approach from~\cite{bozkurt2020}; this results in 
controller strategies with a lower bound on the satisfaction probabilities (see \cite{bozkurt2020} for~details).

\subsection{Efficiency of Learning Controller Strategies}
\vspace{-2pt}

The size of the DRA $\mathcal{A}_{\varphi_\textsc{win}}$ may be double-exponential in the length of $\varphi_\textsc{win}$. %
This %
cannot be prevented in the worst case if $\varphi_\textsc{task}$ is any arbitrary LTL formula. To overcome this, we can restrict the controller tasks to a fragment of LTL, such as Generalized Rabin(1) \cite{ehlers2011}, which can describe {most commonly used} robotic tasks, as well as be efficiently translated into a polynomial-size DRA with a single accepting pair.

Fortunately, any valid IDS mechanism can be expressed as a reachability property, which means if a path satisfies the property, then there must be a prefix of the path such that all the paths having the same prefix also satisfy the property. 
Consider a path $\sigma$, an infinite execution of the game, for which the IDS triggers the alarm -- i.e., $\sigma \models \varphi_\textsc{ids}$. If a finite prefix of $\sigma$ is not enough for the IDS to decide, then even if it can be inferred that adversarial actions have occurred, the IDS could not trigger the alarm in a finite number of~steps. 

The reachability properties form only a small fragment of LTL for which efficient translation to a deterministic finite automata is usually possible \cite{latvala2003}. We can think of such a DFA as a DRA with one accepting state where all the outgoing transitions are self-loops. Therefore, the disjunction of $\varphi_\textsc{ids}$ with $\varphi_\textsc{task}$ in~\eqref{eq:win} only linearly increases the state space of the product game in the length of $\mathcal{A}_{\varphi_\textsc{ids}}$. Additionally, we can utilize the fact that all the incoming transitions to the accept state in the DRA of $\varphi_\textsc{ids}$ are triggered by an action of the attacker. Thus, during the learning phase, whenever a state having a transition to the accepting state is reached, the attacker turn can be skipped. 

Finally, we discuss the effect of choosing different discount factors. As the discount factor $\gamma$ goes to $1$, the convergence rate as well as the stability of the RL algorithms decrease. Hence, we %
start with a smaller $\gamma$ and increase it slowly until the RL algorithm converges to a desired controller strategy. Using smaller $\gamma$ also discourages the controller to wait for a false alarm to be triggered, and thus mitigates the problem discussed in Sec.~\ref{subsection:ids}.
\vspace{-2pt}
\section{Case Studies}
\vspace{-3pt}

For the case studies, we use the %
\textit{CSRL} tool~\cite{csrl2020} %
based on~\cite{bozkurt2020}. %
\textit{CSRL} takes a 2-D labeled grid as a representation of the stochastic game with an LTL formula, and returns a finite-memory controller strategy obtained using the minimax-Q \cite{littman1994} method that follows $\epsilon$-greedy strategies while~learning.

\vspace{-2pt}
\subsection{System Models}
\vspace{-3pt}
We initialized the values of $\epsilon$ (the parameter for $\epsilon$-greedy strategy) and $\alpha$ (the learning rate) to $0.5$, and slowly decreased them to $0.05$ during learning. %
We set the discount factor %
for each case study to be $\gamma{=}0.999$, %
and obtained the controller strategies after $512$K episodes, each %
starting from a random state and stopping after $1$K time steps. 
We used the IDS formula specified in Sec. \ref{section:after_attack} with $m{=}0$ and $n{=}1$:
\begin{equation*}
\varphi_{\textsc{ids}}^{(*)} \hspace{-4pt}=\hspace{-3pt} {\lozenge}\hspace{-1pt}(\texttt{\footnotesize anomaly} {\wedge} {\bigcirc} (\texttt{\footnotesize anomaly} {\wedge}
{\bigcirc} {\lozenge}^{\leq 1}\hspace{-2pt} (\texttt{\footnotesize attack} {\wedge} {\bigcirc} {\lozenge} \texttt{\footnotesize attack}\hspace{-1pt}) \hspace{-1pt})\hspace{-1pt})
\end{equation*}

We used $(7\times 9)$ grid worlds for our planning case studies. %
A robot can move from a cell to an adjacent cell using the controller actions: \textit{North, South, East} and \textit{West}. However, the attacker can observe any action chosen by the controller and replace it with another action. The actions of the controller and attacker are depicted as black and red arrows in the figures, respectively. Due the stochastic %
environment, the robot moves in the intended direction with probability of $0.8$, and sideways with probability of $0.2$ ($0.1$ for each side). The robot stays in the cell, when it tries to move beyond the grid.

\begin{figure*}[!t]
\vspace{4pt}
    \begin{subfigure}{0.3\textwidth}
        \includegraphics[width=\textwidth]{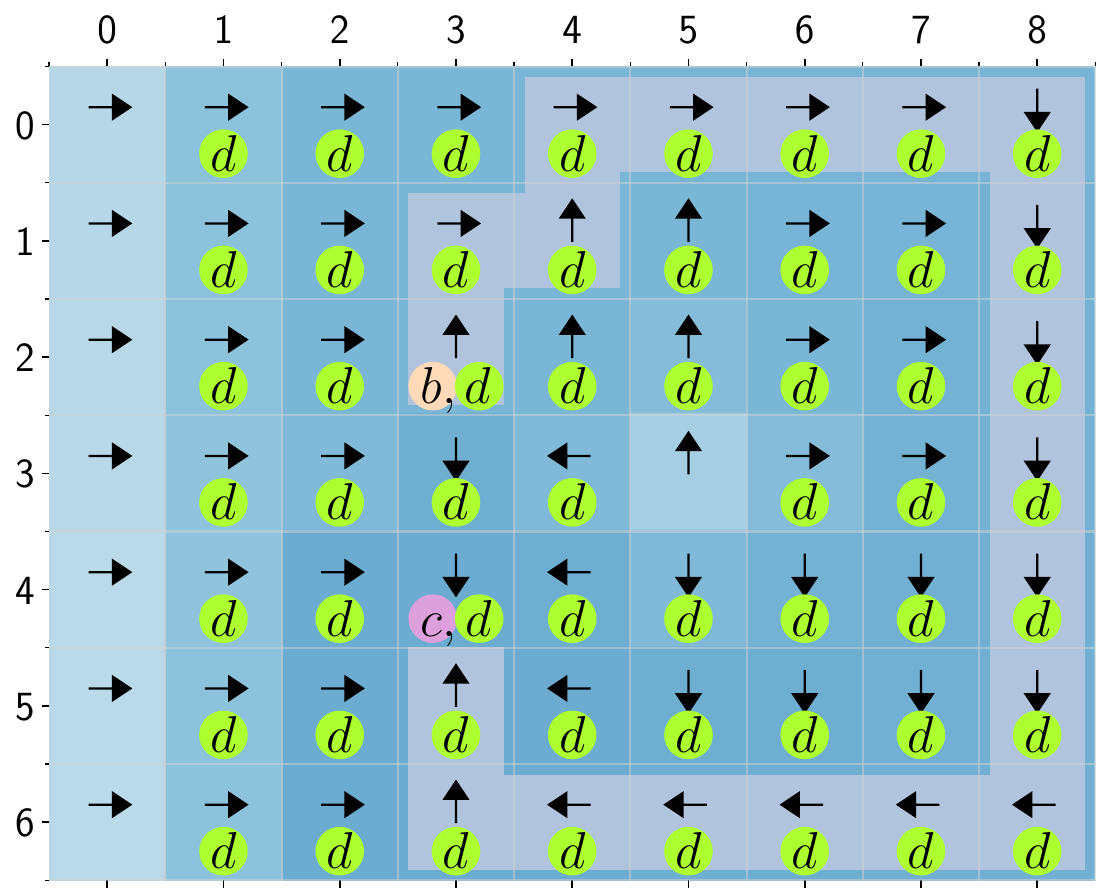}
        \label{fig:coverage_b_to_c}
    \end{subfigure}
    \hspace{2em}
    \begin{subfigure}{0.3\textwidth}
        \includegraphics[width=\textwidth]{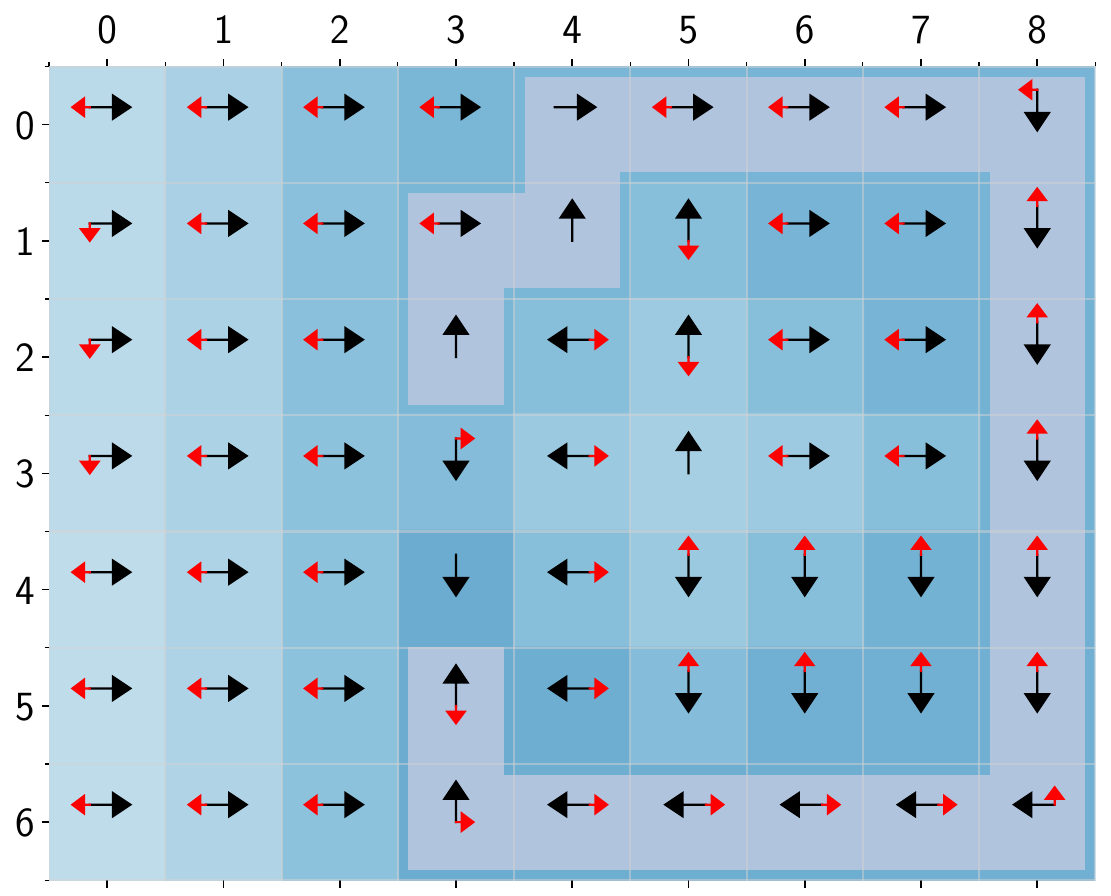}
        \label{fig:coverage_b_to_c_1}
    \end{subfigure}
    \hspace{2em}
    \begin{subfigure}{0.3\textwidth}
        \includegraphics[width=\textwidth]{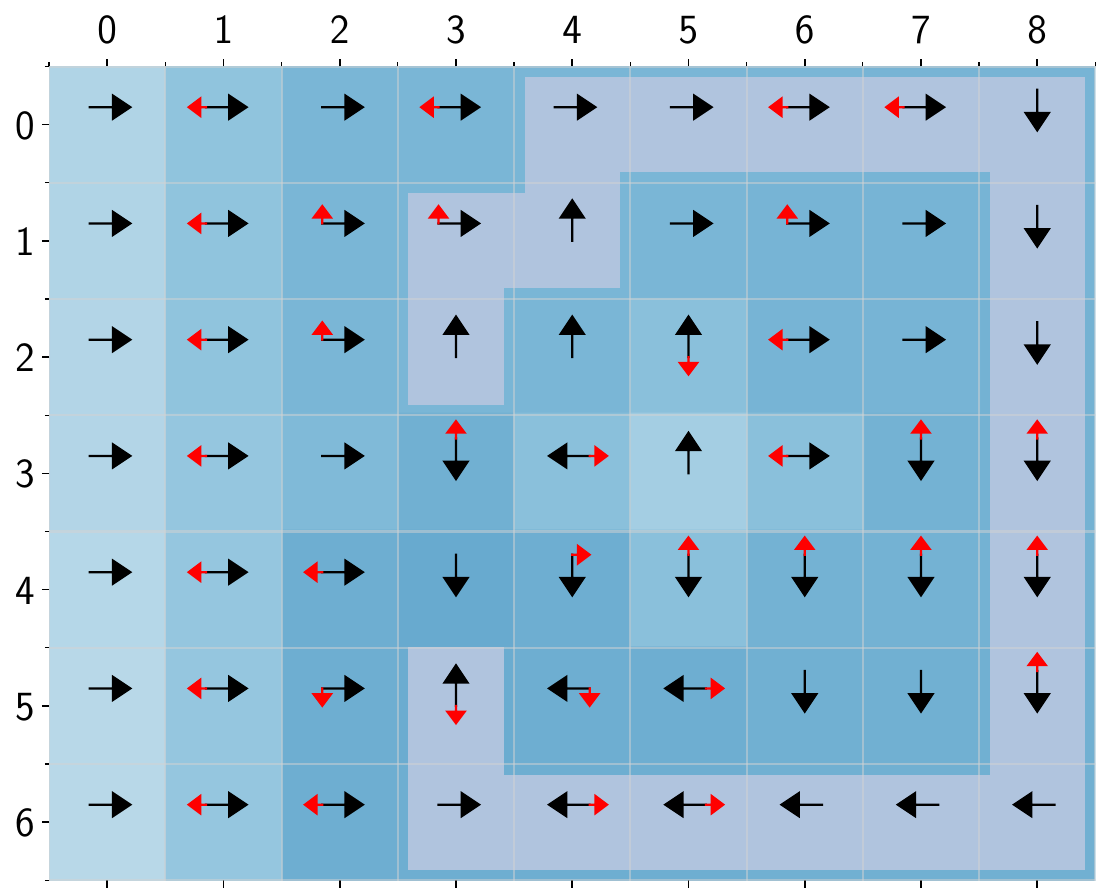}
        \label{fig:coverage_b_to_c_2}
    \end{subfigure}
    \vspace{-19pt}
    \caption{\small Surveillance scenario (from left to right): (a)~The controller strategy from $b$ to $c$ and the cell labels; (b)~The controller and attacker strategies from $b$ to $c$ before any anomaly occurs; (c)~The controller and attacker strategies from $b$ to $c$ after one anomaly.}
    \label{fig:coverage}
    \vspace{-12pt}
\end{figure*}

\begin{figure*}[!t]
    \begin{subfigure}{0.3\textwidth}
        \includegraphics[width=\textwidth]{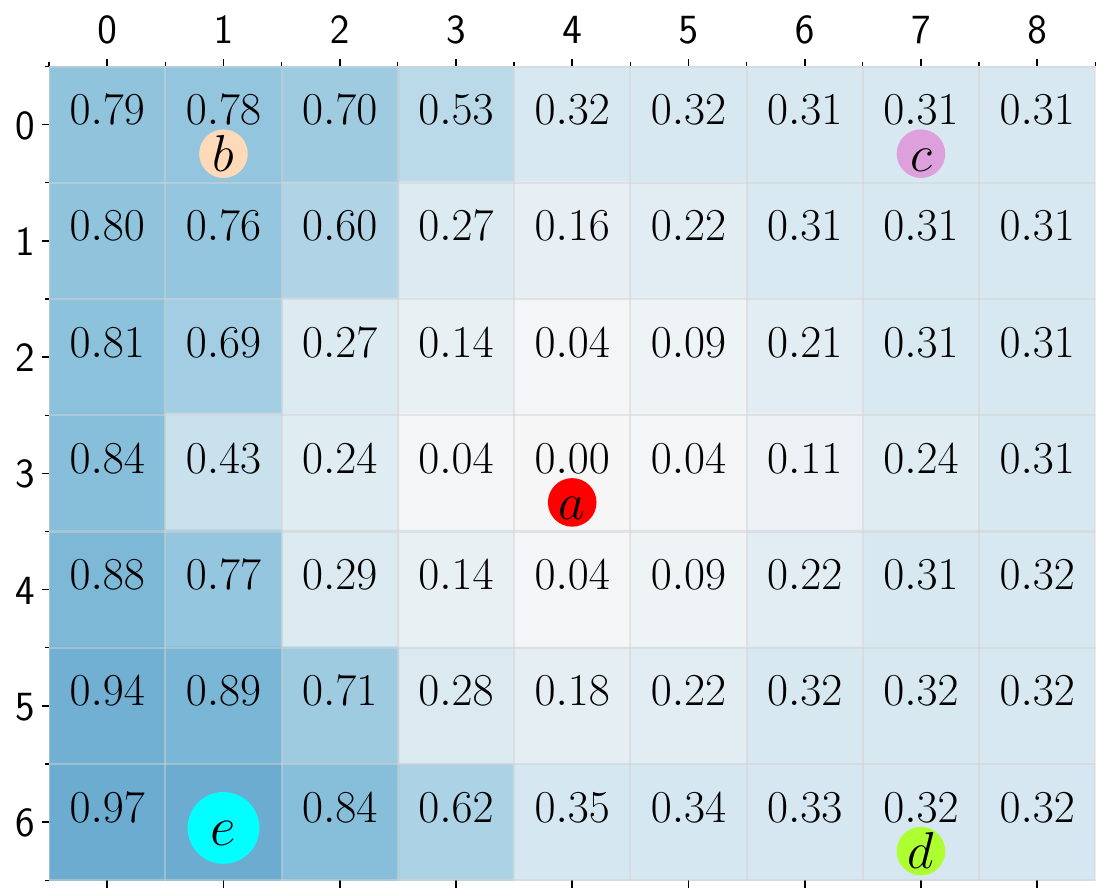}
        \label{fig:sequence_d_to_e}
    \end{subfigure}
    \hspace{2em}
    \begin{subfigure}{0.3\textwidth}
        \includegraphics[width=\textwidth]{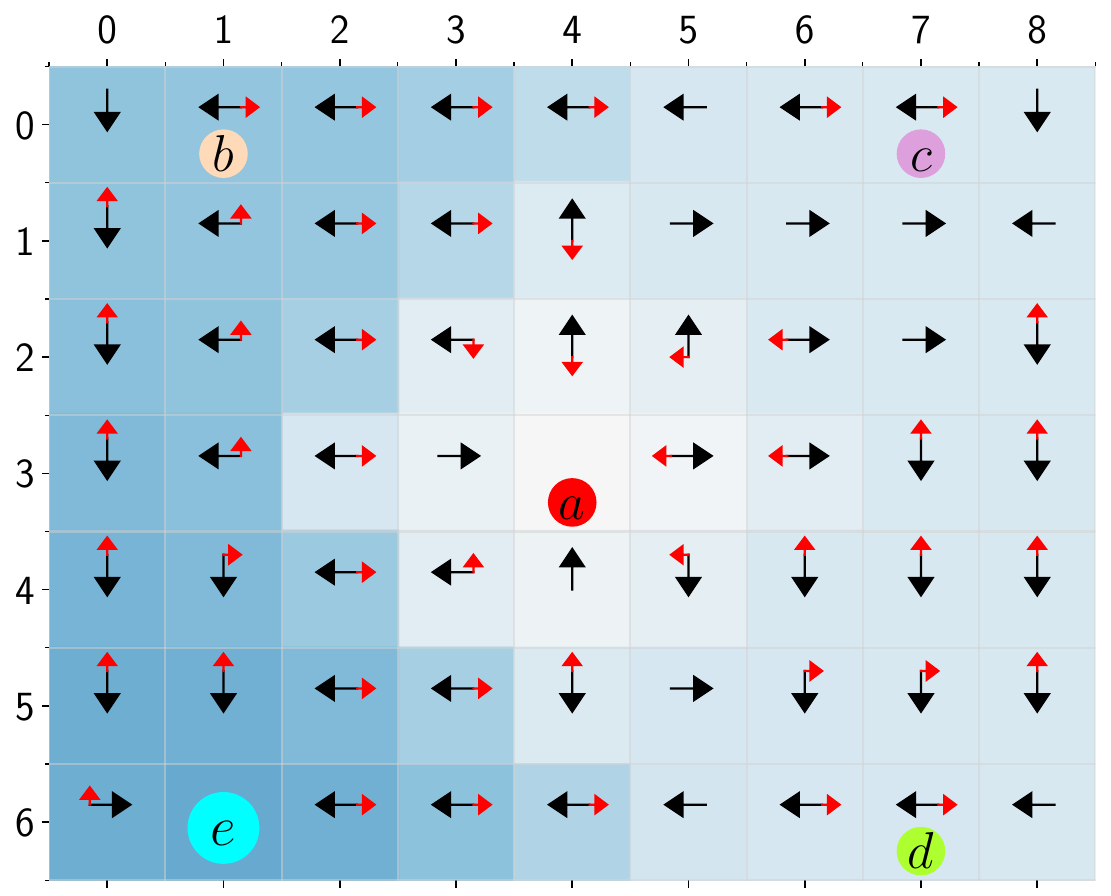}
        \label{fig:sequence_d_to_e_1}
    \end{subfigure}
    \hspace{2em}
    \begin{subfigure}{0.3\textwidth}
        \includegraphics[width=\textwidth]{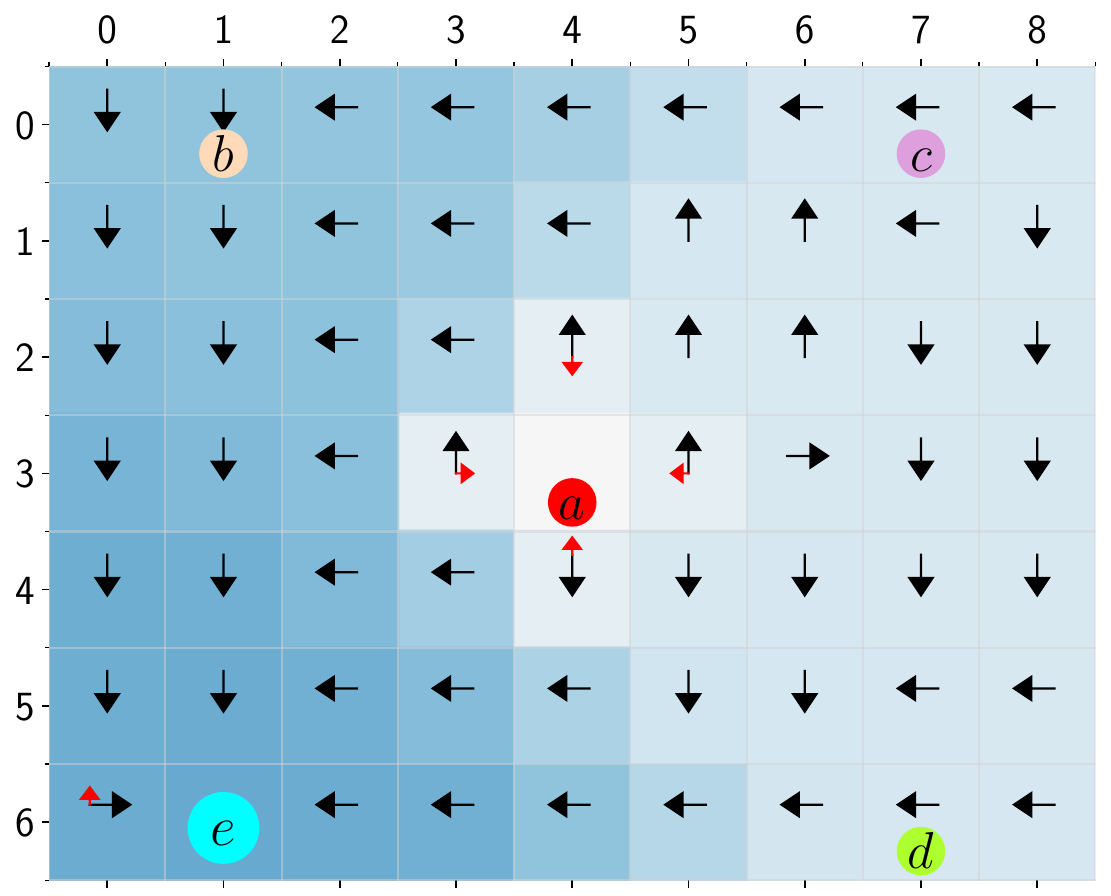}
        \label{fig:sequence_d_to_e_2}
    \end{subfigure}
    \vspace{-19pt}
    \caption{\small Task sequence scenario (from left to right): (a)~The controller strategy from $d$ to $e$ and the cell labels; (b)~The controller and attacker strategies from $d$ to $e$ right after an anomaly occurs; (c)~The controller and attacker strategies from $d$ to $e$ right after~an~alarm.}
    \label{fig:sequence}
    \vspace{-11pt}
\end{figure*}

A grid cell %
corresponds to a controller state in the stochastic game. 
The label of a controller state is displayed in a small circle, %
in the lower part of the corresponding cell in the figures.
After the controller takes an action, the game moves to an attacker state. The attacker observes this transition and either takes the same action or another one, and then the game moves to a stochastic state. For every cell, the controller and attacker actions, we have a different stochastic state in the game. Thus, every stochastic state has a unique parent attacker state and a unique grandparent control~state. 

A stochastic state is labeled with \texttt{\small attack} if and only if the transition from its grandparent state to its parent state, and the one from its parent to itself, are triggered by different actions, i.e., \emph{the attacker modifies the controller action}. %
Even if the \texttt{\small attack} labels are visible to the controller, %
it acts as if they do not exist unless an alarm is triggered; due to the structure of the IDS formula $\varphi_{\textsc{ids}}$ from Sec.~\ref{subsection:ids}, the DRA~of $\varphi_{\textsc{ids}}$ changes its state by an \texttt{\small attack} label only after an alarm.

In a stochastic state, a random transition is made according to the described transition model, to a dummy state first,~then to the controller state of the reached cell. The dummy state is labeled with \texttt{\small anomaly} if the controller state it is connected to is not expected to be reached after the action taken in the great-grandparent controller state. For example, if the controller takes %
\textit{North} and the robot moves east, the dummy state visited is labeled with \texttt{\small anomaly}. We did not explicitly represent the dummy states to keep the learning %
completely model-free; it is enough to make the corresponding DRA transition as if an \texttt{\small anomaly} label is consumed whenever such stochastic transitions occurs during learning.

\vspace{-2pt}
\subsection{Case Study I: Surveillance}
\vspace{-2pt}

In this %
study, the robot tries to repeatedly cover the cells labeled with $b$ and $c$. In addition, after a certain point in time, the robot aims to stay in the \emph{safe} region labeled with $d$, i.e., %
\begin{align}
    \varphi_\textsc{task}^{(1)} = \square \lozenge b \wedge \square \lozenge c \wedge \lozenge \square d. \label{eq:coverage}
\end{align}

The learned %
strategies %
are shown Fig.~\ref{fig:coverage}. For simplicity, we consider only the part of the task where the robot needs to go from $b$ to $c$, but similar results were also obtained for travelling from $c$ to $b$.
Although there is a very short path from $b$ to $c$, the learned controller strategy %
prefers a quite long path. %
There is only one cell between $b$ and $c$, and this cell and all the surrounding cells are in the safe region; however, this is not enough to make the path passing through it secure, because once this cell is visited, the attacker can take two consecutive \textit{East} actions to make the robot visit an unsafe cell with a probability of $0.64$. Thus, %
the robot would be out of the safe region once in a while, %
violating~$\varphi_\textsc{task}$. %

Fig.~\ref{fig:coverage}b and~\ref{fig:coverage}c
show the attacker strategies before an anomaly happens and right after the alarm. %
In three out of the first four cells on the most likely path, the attacker chooses to do nothing (Fig.~\ref{fig:coverage}b).  %
The reason %
is that the attacker does not want to create an unnecessary anomaly in the early part of the path, `saving', in some sense, for the future the ability to create two consecutive anomalies without raising alarm. The attacker strategy in the other parts of the grid %
forces the robot outside the safe region and prevents it from reaching~ $c$.

\vspace{-6pt}
\subsection{Case Study II: Sequence of Tasks}
\vspace{-2pt}
In this scenario, we %
plan for a sequence of tasks %
represented by the labels $b$, $c$, $d$, $e$, %
to be performed in order. %
There is a danger zone labeled with $a$, 
to be avoided
-- i.e., %
\begin{align}
    \varphi_\textsc{task}^{(2)} = \lozenge \Big(b \wedge \lozenge \big(c \wedge \lozenge (d \wedge \lozenge e)\big)\Big) \wedge \square\neg a. \label{eq:sequence}
\end{align}

Here, we present the results and the strategies %
for performing only the last task $e$, i.e., visiting the cell %
$e$; %
however, similar conclusions can be drawn for the other tasks. Fig.~\ref{fig:sequence}a  %
shows the estimated satisfaction probabilities when the IDS has not detected any anomalies. As expected, the probabilities near the danger zone are very low and the probabilities are usually getting lower as the distance to $e$ is increasing.

The satisfaction probabilities in the right part of the grid are significantly lower than the left part. The reason is that while moving from the right part to the left, the minimum number of cells between the robot and $a$ can be at most two, e.g., at $(6,4)$, which allows the attacker to take three consecutive actions, e.g., $3\times$\textit{North} from $(6,4)$, to drag the robot into the danger zone with a probability larger than $0.5$. %
If all the attacks are successful, i.e., the robot moves in the direction the attacker desired, after the first two actions the IDS raises an alarm. After the alarm, %
although the attacker knows (s)he will be detected, (s)he %
attacks again if the robot is next to $a$, as there is a high probability ($0.8$) that the robot moves into the danger zone, %
making the attacker the winner.

Fig.~\ref{fig:sequence}b %
shows the controller and attacker strategies after an anomaly occurs. %
Again, the attacker is more passive in the upper right part of the grid, as (s)he does not want to trigger an alarm when the controller is far from reaching $e$. Fig.~\ref{fig:sequence}c %
shows the strategies when the IDS is in the high-alert mode, in which any attack is immediately detected. The attacker takes an action only in five cells. In the cell at the bottom-left corner, (s)he makes one last attempt to prevent the controller from reaching $e$ because even if (s)he is detected, if (s)he does nothing the controller wins the game with a probability slightly lower than $1$. The other cells are the ones 
next to the danger zone, where %
the attacker sacrifices stealthiness for the high probability that the robot ends up in the danger zone.
\vspace{-1pt}
\section{Conclusion}
\vspace{-2pt}
We studied planning for temporal logic tasks in an unknown stochastic environment in the presence of actuation attacks. %
We formulated the interaction between the controller and stealthy attacker as a stochastic game, where the attacker, knowing the intrusion detection systems (IDS), the task, and the controller; aims to undermine the task while remaining stealthy. %
We then show this planning problem can be solved using model-free reinforcement learning without knowledge of the environment model. Our case studies showed the applicability of our method for security-aware robotic~planning.

\bibliographystyle{unsrt}
\bibliography{references,yu,CPSL@DukePapers}

\end{document}